\documentclass{article}
\usepackage[utf8]{inputenc}
\usepackage{amsmath}
\usepackage{amsthm}
\usepackage{amssymb}
\usepackage[round]{natbib}
\usepackage{hyperref}
\hypersetup{colorlinks=true,allcolors=[rgb]{0,0,1}}
\usepackage{geometry}
\newtheorem{theorem}{Theorem}[section]% meant for sectionwise numbers
\newtheorem{definition}[theorem]{Definition}

\newcommand{\indep}{\perp \!\!\! \perp}
\title{Choosing with unknown causal information: Action-outcome probabilities for decision making can be grounded in causal models}
\author{Mauricio Gonzalez Soto$^{1,2}$ \and David Danks$^3$ \and Hugo J.  Escalante Balderas$^2$ \and Luis E. Sucar Succar$^2$}
\date{%
    $^1$FSG NeuroInformatics, Faculty of Computer Science, University of Vienna. \\[2ex]%
    $^2$Coordinaci\'on de Ciencias Computacionales, Instituto Nacional de Astrof\'isica \'Optica y Electr\'onica (INAOE), M\'exico\\[2ex]%
    $^3$Halicioglu Data Science Institute and Department of Philosophy, University of California, San Diego.\\[2ex]%
    \today
}

\begin{document}

\maketitle

\begin{abstract}%
Decision-making under uncertainty and causal thinking are fundamental aspects of intelligent reasoning. Decision-making has been well studied when the available information is considered at the associative (probabilistic) level. The classical Theorems of von Neumann-Morgenstern and Savage provide a formal criterion for rational choice using associative information: maximize expected utility. There is an ongoing debate around the origin of probabilities involved in such calculation. In this work, we will show how the probabilities for decision-making can be grounded in causal models by considering decision problems in which the available actions and consequences are causally connected. In this setting, actions are regarded as an intervention over a causal model. Then, we extend a previous causal decision-making result, which relies on a known causal model, to the case in which the causal mechanism that controls some environment is unknown to a rational decision-maker. In this way, action-outcome probabilities can be grounded in causal models in known and unknown cases. Finally, as an application, we extend the well-known concept of Nash Equilibrium to the case in which the players of a strategic game consider causal information.
\end{abstract}

\section{Introduction}
An important aspect of acting in the world is being able to make decisions under uncertain conditions: Which route should I use to get to work? Could there be traffic? If one wishes to avoid certain incoherent, inconsistent behavior, then the natural and most well-known idea is to balance between how desired an object is, and how available is it. Such balancing is to be done via the calculation of the expected utility \citep{bernoulli1954exposition,von1944theory,savage1954the}, where how desired an object is, is measured in terms of the \textit{utility} it produces; and how available is it is encoded in terms of probabilities. Such probabilities are required to either be objectively known or subjectively (\textit{internally}) defined by a decision maker. There is no standard agreement in the decision-making literature on where such probabilities must come from. 

Acting in the world is conceived by human beings as \textit{causally intervening} on it, and it is known that humans are also able to learn and use causal relations while making choices \citep{lagnado2007beyond,hagmayer2007causal,hagmayer2008causal,hagmayer2009decision,hagmayer2013repeated,hagmayer2017causality}. Thinking in terms of causes and effects is an everyday task and in fact, \textit{causal reasoning} is to be found at the very core of our minds since we constantly ask \textit{why}: Why do we get sick? Why does a drug work? \citep{tversky1977causal,spirtes2000causation,waldmann2013causal,danks2014unifying,lake2017building,pearl2018why,neil2019causality}.

If, as argued, another important aspect of acting in the world is making choices, then taking causal information as a basis is fundamental for decision-making. Therefore, it is a natural question \textit{how to formalize decision making when causal information (known or unknown) is present?} Answering the such question is relevant given the importance of causal relations as well as reasoning in everyday life and in science \citep{spirtes2000causation,pearl2018why}. Given that human beings actually use causal information while making choices \citep{tversky1980causal}, and the importance of decision-making results based on associative information, it is desirable to have an explicit and computationally implementable criterion for decision-making, which also addresses the question of the origin of probabilities for decision making in a sensitive, reasonable way.

The previous question has been already considered by \cite{nozick1969newcomb,lewis1981causal,joyce1999foundations,eells2016rational, stern2017interventionist} as well as by \cite{pearl2009causality} who provides an optimality criterion for decision making under causal-controlled uncertainty \textit{when the causal mechanism which controls the environment is known} by the decision maker \citep{pearl2009causality}.

In this paper, we will provide a decision-making criterion for \textit{unknown} causal information, and such criterion has the form of maximization of expected utility; thus showing that action-outcome probabilities for decision-making can be grounded in causal models.

The remainder of the paper goes as follows: In Section \ref{Classical_Decision_Making} we review the basics of Classical Decision Making: first we give some basic definitions; then we review the decision making results from von Neumann-Morgenstern and Savage,  where we discuss the General Expected Utility Criterion, and finally we discuss about the problem around the origin of the action-outcome probabilities. In Section \ref{Causation}, the basics of Causation and Causal Graphical Models, as well as the basic definitions required for Causal Decision Making. In Section \ref{causal_decision_making} we describe some fundamental previous results for Causal Decision Making, and in particular our main result in Section \ref{main_result}, finally in Section \ref{causal_games} we describe an application into the domain of Game Theory, where we are able to define a Causal Nash Equilibrium.

\section{Classical Decision Making}{\label{Classical_Decision_Making}}
A Decision Problem under Uncertainty is the mathematical model of a situation in which an agent must choose one out of many available actions with uncertain consequences which depend on different, possibly unknown, factors \citep{bernardo2000bayesian,robert2007bayesian,gilboa2009decision}. Such consequences are assumed to produce \textit{satisfaction} in the decision maker, and this satisfaction is represented by a \textit{preference relation}, defined over actions, and denoted by $\succeq$, where $a_1 \succeq a_2$ is read as $a_1$ being preferred to $a_2$, in terms of the consequences that would be obtained by choosing action $a_1$ instead of action $a_2$, if these were the only two available options and given the current knowledge the agent has. The agent is assumed to be \textit{rational}; that is, it is assumed that the preference relation satisfies the so-called rationality or coherence axioms described in Appendix \ref{appendixA}.

\citeauthor{von1944theory} gave an answer for how to make choices if rational preferences are assumed, utilities are unknown, and the stochastic relation (i.e., probabilities of events) between actions and outcomes are considered as objective and given to the decision maker: maximize expected utility with respect to a utility function whose existence is guaranteed \citep{von1944theory}.

If probabilities are not known, then \citeauthor{savage1954the} showed that a rational decision maker must choose \textit{as if} she is maximizing the expected utility using a \textit{subjective} probability distribution \citep{savage1954the}. 

Such theorems provide a formal criterion for associative, also called evidential, decision-making if rationality is assumed: maximize expected utility.

Other decision-making theories exist, such as Kahneman and Tversky's Prospect Theory \citep{kahneman1979prospect}, Gilboa's Case-Based Decision Theory \citep{gilboa1995case}, among others that are out of the scope of this work, since such theories defy the \textit{classical} notion of rationality, which is the one we are adopting in this work.

In Section \ref{related_work} we make reference to another theory, Joyce's Causal Decision Making \citep{joyce1999foundations} as well as Pearl's criterion for known causal information. For further details on classical (non-causal) decision making, see \cite{bernardo2000bayesian} and \cite{gilboa2009decision}. For further applications of Decision Making in other fields, such as Quantum Mechanics, see \cite{wallace2012emergent,berkovitz2012world}.

\subsection{Preliminary Definitions}
In this Section, we formally state the definitions required for a precise formulation of what follows. The notation and setting we use here is the one used by \cite{bernardo2000bayesian}. In Definitions \ref{def_UE} and  \ref{def_DPU}, we consider a non-empty set $\Omega$ and a countable set $\mathcal{A}$ of available actions; for each action $a_i \in \mathcal{A}$, a partition $(E_j)_{j \in J(i)}$ of $\Omega$ and a set of consequences $(C_j)_{j \in J(i)}$. Let $\mathcal{E}$ be the $\sigma$-algebra generated by the union of every $(E_j)_{j \in J(i)}$, and let $\mathcal{C}$ the union of all $(C_j)_{j \in J(i)}$ over $i$. \cite{bernardo2000bayesian} derives the existence of a subjective probability measure from a set of \textit{coherence axioms}, according to which a decision maker has some mechanism of quantifying uncertainty in terms of real numbers within the $[0,1]$ interval. \\
\\
\textbf{Note:} In the context of Chapter 2 of %\citeauthor{bernardo2000bayesian}
\cite{bernardo2000bayesian}, from a preference relation $\succeq$ defined over actions of the form $\{c_{ij} \mid E_{ij}, j \in J(i) \} $, and the coherence axioms stated in the Appendix, further relationships $\succeq$, $\succ$ as well as $\sim$ can be defined both over consequences, as well as \textit{over events}, where $E \succeq F$ is to be read as considering $E$ more likely than $F$, as well as $E \succeq_G F$, which is the \textit{conditional} likelihood relation, but such development is beyond the scope of this work. Having said this, we can reduce the notation $a_i = \{ c_{ij} \mid E_{ij}, \textrm{ } j \in J(i) \}$ to simply $\{ c_j \mid j \in J \}.$ But it must be kept in mind that the index set $J$ depends on $i$ as well as $c_j$ and each $E_j$. We note that each $a_i = \{ c_{ij} \mid E_{ij}, \textrm{ } j \in J(i) \}$ links, or maps, a \textit{partition} of uncertain events $\{ E_j : j \in J \}$ to a corresponding set of consequences $\{ c_j : j \in J \}. $ Thus bridging this notation to the original formulation of Savage's Theorem, where the objects of choice are functions from states, which are defined by \cite{gilboa2009decision} as \textit{...an exhaustive list of all scenarios that might unfold...}, to outcomes \citep{bernardo2000bayesian,gilboa2009decision}.\\
\\
We will now first consider the general setting for a decision problem: an uncertain environment.

\label{definitions}
\begin{definition}
\label{def_UE}
Let $\Omega$ be a non-empty set. An \textbf{uncertain environment} is a tuple $(\Omega, \mathcal{A},\mathcal{C},\mathcal{E})$. Where $\mathcal{A}$ is a non-empty set of available actions, $\mathcal{C}$ a set of consequences, and $\mathcal{E}$ is a $\sigma$-algebra of events over $\Omega$. 
\end{definition}

If we consider \textit{the preferences} of some decision maker over the set of actions of some uncertain environment, then we have a Decision Problem under Uncertainty \citep{bernardo2000bayesian}.

\begin{definition}
\label{def_DPU}
A \textbf{Decision Problem under Uncertainty} is an uncertain environment $(\Omega, \mathcal{A},\mathcal{C},\mathcal{E})$ endowed with a preference relation $\succeq$ over $\mathcal{A}$. 
\end{definition}
As previously mentioned, we can identify, as done in \cite{bernardo2000bayesian}, consequences as actions by simply writing $c = \{ c \mid \Omega \}$ and say that consequence $c_1$ is preferred over consequence $c_2$, denoted as $c_1 \succeq c_2$, if action $\{ c_1 \mid \Omega \}$ is preferred over action $\{ c_2 \mid \Omega \}$.
\begin{definition}
A Decision Problem under Uncertainty is said to be \textbf{bounded} if there exists a pair of consequences $c_\ast$ and $c^\ast$ such that for every $c \in \mathcal{C}$, $c^{\ast} \succeq c \succeq c_\ast$.
\end{definition}

\begin{definition}
A Decision Problem under Uncertainty is said to be \textbf{finite} if the set $\mathcal{A}$ of available actions is finite.
\end{definition}

\subsection{Objective Probabilities: von Neumann-Morgenstern Theory}
\label{vNM-M}
Consider the following scenario: throwing a die. In this scenario, using DeFinetti's Representation Theorem \citep{finetti1992foresight,schervish1995theory}, we can model such throws as independent realizations of a random variable that has a known fixed law. We can think of such a law as being \textit{objective}. The \citeauthor{von1944theory} (vNM) Theorem  considers a scenario of \textit{decision under risk} with rational preferences; this is, rationally choosing between alternatives with uncertain outcomes with known, objective, probabilities. 

Formally, using the notation by \cite{gilboa2009decision}, we consider a set $X$ of available options. Let $L$ be the set of lotteries with finite support over $X$. The objects of choice are elements $l \in L$, which are known to the decision maker; we represent the decision maker's preferences by a preference relation $\succeq \subseteq L \times L$ that  satisfies being complete, transitive, continuous, and a condition called \textit{independence}. This family of conditions is called \textit{von Neumann-Morgenstern rationality axioms} as described by \cite{gilboa2009decision} and \cite{schervish1995theory}.

\begin{theorem}[von Neumann-Morgenstern]{\label{vNM}}
A preference relation $\succeq \subseteq L \times L$ where $L$ is a set of lotteries with finite support over a set $X$ satisfies the von Neumann-Morgenstern rationality axioms if and only if there exists a function $u: X \to \mathbb{R}$ such that for every $P, Q \in L$ we have that
\begin{equation}
P \succeq Q \textrm{ if and only if } \sum_{x \in X} P(x) u(x) \geq \sum_{x \in X} Q(x) u(x). 
\end{equation}
\end{theorem}
The theorem states that \textbf{if a rational decision maker knows the probabilities of obtaining a certain outcome}, then she must choose \textit{as if} maximizing the expected value of some function $u$ whose existence is guaranteed by Theorem \ref{vNM}. See \cite{gilboa2009decision} for details on the proof.

\subsection{Subjective Probabilities: Savage}
\label{savage_theorem}
Thanks to DeFinetti's Theorem, we can think of objective probabilities as long-term frequencies, but the question remains open regarding inquiries of the type: \textit{what is the probability that John Doe was born on May 1st?} It is not the case that John Doe sometimes is born one day and sometimes born again some other day. In this context, we may talk of probabilities as specifying \textit{beliefs} an agent has over some the occurrence or not of some event. Savage's Theorem studies decision-making in this context; this is, if a rational decision-maker is uncertain about the probabilities of obtaining certain outcomes and does not have a precise, objective quantification of her preferences (utility function), then it is Savage's Theorem \citep{savage1954the}, which gives a formal decision criterion. 

Savage's result extends von Neumann-Morgenstern Theorem since it considers the case in which a rational decision maker knows neither her utility function nor the probabilities to be used in order to obtain the expected values required for making choices according to the von Neumann-Morgenstern Theorem \citep{gilboa2009decision}. 

\subsubsection{General Expected Utility Principle (Savage)}
We now state the general Expected Utility Principle as found in \citep{bernardo2000bayesian}. Further details can be found in \citep{savage1954the,anscombe1963definition,kreps1988choice,schervish1995theory,robert2007bayesian,gilboa2009decision}.

\begin{theorem}[Expected Utility Principle, \cite{savage1954the,bernardo2000bayesian}]{\label{savage}}
In a finite, bounded Decision Problem under Uncertainty $(\mathcal{A}, \mathcal{C}, \mathcal{E}, \succeq)$, the preference relation $\succeq$ satisfies the coherence rationality axioms if and only if there exists: A \textit{probability measure} $P$, called a \textit{subjective probability}, that associates with each uncertain event $E \in \mathcal{E}$ a real number $P(E)$ and a utility function $u : \mathcal{C} \to \mathbb{R}$ such that it associates each consequence with a real number $u(c)$. Such that for $a_1$ and $a_2$ actions in $\mathcal{A}$, and any $G \neq \emptyset$
\begin{eqnarray*}
 &a_1 \succeq_G a_2&\\
 & \textrm{ if and only if }&\\
 & \sum_{j \in J(a_1)} u(c_j) P(E_j) \geq \sum_{j \in J(a_2)} u(c_j) P(E_j) &
 %& \textrm{ which can be written in a simple way as }&\\
 %&\mathbb{E}_P[\bar{u}(a_1)] \geq  \mathbb{E}_P[\bar{u}(a_2)],&
\end{eqnarray*}
\end{theorem}
%Where $\bar{u}(a_i) = \sum u(c_j) P(E_j)$ in the representation of $a_i$ as $\{c_j \mid E_j ; j \in J\}$.\\
This theorem states that if a rational decision maker does not know the precise probabilities of outcomes given that a certain action has been taken, then she must choose \textit{as if} having in mind a probability assignment to the uncertainties in her environment and use such probabilities to calculate the expected utility with respect to a subjective utility function that represents her preferences. This result also gives a precise definition of \textit{subjective probability} as a quantification of uncertainty which is used to make good decisions \citep{gilboa2009decision}. See \citep{hens1992note,bernardo2000bayesian}, and \citep{gilboa2009decision} for further details. See \citep{ellsberg1961risk,tversky1975critique,tversky1989rational,binmore2008rational,gilboa2009always} and \citep{cite-key} for critiques of the coherence axioms.

\subsection{Origins of outcome-action probabilities}
The previous theorems, both vNM and Savage's, have the limitation of being stated only in terms of associative information (probabilities), which leaves open the question of the origin of such probabilities. We know that vNM provides utilities given probabilities, and Savage provides both utilities and probabilities given the coherence axioms, but what do such probabilities mean and what is their origin?

Interpreting the action-outcome conditional probabilities, either as causal or observational, is exactly the issue at the heart of the debate between causal and evidential decision theorists, see \cite{joyce1999foundations}. It is recognized in the literature that the source of these conditional probabilities should be explained \citep{peterson2017introduction}, but there is substantial debate about their source, see for example \cite{binmore2008rational,gilboa2009decision,eells2016rational}. We need a further justification to show that the right action-outcome conditional probabilities can be grounded in causal models. The objective of this paper is to provide that basis. In the next section we state some basic facts on causality and why $P(\cdot \mid x)$ is different from $P(\cdot \mid do(x))$ for some value $x$.

\section{Causation}{\label{Causation}}
The concept of Causality deals with regularities found in a given environment (context) that are stronger than probabilistic (or associative) relations in the sense that a causal relation allows for evaluating a change in the \textit{consequence} given a change in the \textit{cause}. This is known as an \textit{intervention}, which is different from observing and consists of a change in the joint distribution of the variables, which is performed by forcing the value of some variable to a specific value. Causal reasoning is able to deal with changes in the data-generating distributions, while observational reasoning does not. In a non-causal world, patients would avoid going to the doctor in order to avoid being sick \citep{spirtes2000causation,pearl2009causality,koller2009probabilistic,pearl2018why}.

In this work, the \textit{manipulationist} interpretation of Causality is adopted \citep{woodward2005making}. The main paradigm is clearly expressed by \citeauthor{campbell1979quasi} as \textit{manipulation of a cause will result in a manipulation of the effect} \citep{campbell1979quasi}. Consider the following example from \citeauthor{woodward2005making}: manually forcing a barometer to go down won't cause a storm, whereas the occurrence of a storm will cause the barometer to go down \citep{woodward2005making}. 

Here, we take the formal definition of Probabilistic Causality given by \citeauthor{spirtes2000causation} as a working definition for the notion of Causation. Similar descriptions of the manipulationist approach were described by \cite{holland1986statistics}. 
Causal inference tools, such as Pearl's do-calculus allow finding the effect of an intervention in terms of probabilistic information when certain conditions are met \citep{pearl2009causality}.

\subsection{A definition of Causality}
Causality is a \textit{stochastic} relation between \textit{events} within a probability space; this is, some event (or events) \textit{causes} another event to occur, \citep{spirtes2000causation}. 

\begin{definition}{\label{causal_relation}}
Let $(\Omega, \mathcal{F}, \mathbb{P})$ be a probability space, and consider a binary relation $\to \subseteq \mathcal{F} \times \mathcal{F}$ which is: Transitive: If $A \to B$ and $B \to C$ for any $A, B, C \in \mathcal{F}$ then $A \to C$. Irreflexive: For all $A \in \mathcal{F}$ it doesn't hold that $A \to A$. Antisymmetric: For $A,B \in \mathcal{F}$ such that $A \neq B$ if $A \to B$ then it doesn't hold that $B \to A$.
\end{definition}

We consider an extra pair of conditions. The first one, known as Causal Sufficiency, is about the \textit{nature} of the model: for any variables $X, Y$ in the model $\mathcal{G}$, there are no common causes of $X, Y$ \textit{outside} of the model $\mathcal{G}$ \citep{spirtes2000causation,pearl2009causality,sucar2015probabilistic}.\footnote{Alternately, if noise terms are specified in $\mathcal{G}$ to capture unmeasured causes, then Causal Sufficiency can be written as “there are no unmodeled causes of $X$.”} Common causes that are uncaused by other factors in $\mathcal{G}$ are the usual emphasis with Causal Sufficiency, but it also excludes some unobserved \textit{intermediate} events (e.g., if $A \rightarrow B \rightarrow X$ and $B \rightarrow Y$, then $B$ must be observed). Causal Sufficiency thus implies that the causal connections in $\mathcal{G}$ do not share unmodeled mechanisms \citep{spirtes2000causation}. 

The second required condition is that an intervention on a particular variable or event $T$ will change the value of $T$ (and so can “break” the causal influences on $T$), but do not otherwise affect the causal mechanisms in $\mathcal{G}$. This assumption, called \textit{Invariance} by \cite{woodward2005making}, more specifically implies that $T$ still has the same effects as before, even though the joint distribution of the variables is reconfigured \citep{woodward2005making}. 

These conditions are required in an axiomatic fashion, so we do not discuss them further here.

\subsection{Representation into a Directed Acyclic Graph}
The causal relations, defined between events, contained in $\to$ can be summarized in a graph $G=(V,E)$ in the following way: If $A \to B$ then the graph must contain a node $A \in V$ representing $A$, a node  $B \in V$ representing $B$ and a directed edge $e \in E$ connecting the respective nodes in the direction of the causal relation.

%\begin{proposition}
%Given a causal relation $\to$ as in Definition \ref{causal_relation} then the graph that is obtained by considering nodes for events and edges for the causal relations as previously described is a Directed Acyclic Graph.
%\end{proposition}
%The proof is immediate from the definition of Causation. Further details in \cite{spirtes2000causation}.

Notice that since the graph is finite, there exist some nodes that do not have causes, which are called \textit{exogenous}. If an event $A$ is caused by some other event, then we say it is \textit{endogenous} and we denote the set of its causes as $Pa(A)$. It is proven by \citeauthor{kiiveri1984recursive} that at least one exogenous node exists in a causal graph \citep{kiiveri1984recursive}.

\subsection{Causal Graphical Models}
A Causal Graphical Model (CGM) consists of a set of random variables $\mathcal{X}=\{ X_1,...,X_n \}$, and a Directed Acyclic Graph (DAG) $\mathcal{G}$ whose nodes are in correspondence with the variables in $\mathcal{X}$ and whose edges represent relations of cause-effect in the sense that their realizations correspond to the events encoded by the causal relation\citep{koller2009probabilistic,sucar2015probabilistic}. Also, the model is enriched with an operator called $do()$ which is a functional defined over graphs, and whose action is described as follows: given $\mathbf{X} \subseteq \mathcal{X}$ and $\mathbf{x} = \{ x_{i_1}, x_{i_2}, ... , x_{i_j} \}$ an element of the set of all possible values of the variables belonging to $\mathcal{X}$, $Val(\mathcal{X})$ the action $do(\mathbf{X} = \mathbf{x} )$ corresponds to assigning to each $X_j \in \mathbf{X}$ the value $x_{i_j}$ and to delete every incoming edge into the node corresponding to each $X_j$ in the graph $\mathcal{G}$.
In the context of CGMs, to apply the $do()$ operator over a variable (or set of variables) is called as an \textit{intervention} over the variable. It is this interventional operator which separates associative models from causal models \citep{pearl2009causality,koller2009probabilistic,sucar2015probabilistic}.

It is required that the probability distribution that results from an intervention over a variable is Markov compatible with the graph; this is, the resulting interventional distribution is equivalent to the product of the conditional probability of every variable given its parents in the intervened graph \citep{sucar2015probabilistic}.

\subsection{Causal Environments and Causal Decision Problems}{\label{causal_problems}}
We consider decision-making using causal information. In this section we define a Causal Environment to be an \textit{uncertain environment}, as defined in Section \ref{definitions}, with the extra condition that there exists a CGM $\mathcal{G}$ which controls it in the following sense: there exists a causal relation between available actions and consequences in the sense that any chosen action will stochastically \textit{cause} a consequence. The role of the CGM is to encode all of the causal relations present in the environment, not only between actions and consequences but also between any other variables in the environment. 

\begin{definition}{\label{causal_environment}}
A \textbf{Causal Environment} is a tuple $(\Omega, \mathcal{A},\mathcal{G},\mathcal{C},\mathcal{E})$ where\\ $(\Omega, \mathcal{A},\mathcal{C},\mathcal{E})$ is an uncertain environment and $\mathcal{G}$ is a CGM such that the set of uncertain events $\mathcal{E}$ correspond to the different realizations of the variables in $\mathcal{G}$ as well as the possible ways that the variables are related one with each other.
\end{definition}

\begin{definition}{\label{causal_decision_problem}}
We define a \textbf{Causal Decision Problem} (CDP) as a tuple $(\Omega, \mathcal{A}, \mathcal{G},\mathcal{E},\mathcal{C},\succeq)$, where $(\Omega, \mathcal{A}, \mathcal{G},\mathcal{E},\mathcal{C})$ is a Causal Environment and $\succeq$ is a preference relation defined over actions. 
\end{definition}

For the CGM in a CDP, we distinguish two particular variables: one corresponding to the available actions, and one corresponding to the produced (caused) outcome. We are considering that only one variable can be intervened upon and that the values of such variable represent the actions available to the decision maker; i.e., the value forced upon such variable under an intervention represents the action taken by the decision maker. 

The intuition behind the definition of a Causal Decision Problem is this: a decision maker chooses an action $a \in \mathcal{A}$, which is automatically fed into the model $\mathcal{G}$, which outputs the \textit{causal outcome} $c \in \mathcal{C}$. The definitions of a Bounded and Finite Causal Decision Problem extend in an analogous way from the classical definition. 

\section{Causal Decision Theory}{\label{causal_decision_making}}
Now that we have addressed basic notions of Classical Decision Making and Causation, we move on to Causal Decision Theory, where actions and outcomes are causally related. In this Section, we will recall some related and previous work done in Causal Decision Theory, with particular emphasis on decision-making results by Judea Pearl, in which causal information is assumed to be known to a decision-maker. Then, we further generalize this result to the case of \textit{unknown} causal information.

\subsection{Related Work}
\label{related_work}
According to \citeauthor{joyce1999foundations}'s formulation of Causal Decision Theory, a decision maker must choose whatever action is more likely to (causally) produce the desired outcome while keeping any beliefs about causal relations fixed \citep{joyce1999foundations}. This is resumed in Stalnaker's equation \citep{stalnaker1968}:
\begin{equation}
u(a)=\sum_{x} P(a \square \to x)u(x),
\end{equation}
where $a \square \to x$ is to be read as \textit{if the decision maker does} $a$ \textit{then} $x$ \textit{would be the case} \citep{gibbard1978counterfactuals,kleinberg2013causality}. Lewis' and Joyce's work captured the intuition that causal relations may be used to control the environment and to predict what is caused by the actions of a decision-maker. We will refine the $\square \to$ operator by an explicit way of calculating the probability of causing an outcome by doing a certain action in terms of Pearl's do-calculus. \citeauthor{heckerman1995decision} provides a framework for defining the notions of cause and effect in terms of decision theoretical concepts and gives a theoretical basis for a graphical description of causes and effects, such as the causal influence diagrams introduced by \citeauthor{dawid2002influence} \citep{dawid2002influence}. \citeauthor{heckerman1995decision} gave an elegant definition of causality but did not address how to actually make choices using causal information \citep{heckerman1995decision}. \citeauthor{dawid2012decision} presents a decision-theoretic approach to causal inference in which a decision maker must take into account how alternatives compare against each other in terms of the \textit{average causal effect}, the such approach uses the well-known influence diagrams \citep{dawid2002influence,dawid2003causal} in order to derive formulas that allow an explicit calculation of the average causal effect \citep{dawid2012decision}. Influence diagrams have the ability to express both intervention variables and chance variables in a single graphical structure. %in such a way that the standard techniques for probabilistic DAGs still apply. 
An optimality criterion for sequential interventions is obtained by \citeauthor{dawid2008identifying} by maximizing the expectation of outcomes \citep{dawid2008identifying}. 

On the other hand, J. Pearl proposes the following criterion for decision-making: Consider a rational decision-maker who faces a causal environment in which she knows the causal model controlling the relation between her actions and outcomes. She can use the known causal model in order to find the probabilities of \textit{causing} the desired outcome given she takes a certain action. The following theorem is found in Section 4.1 of \cite{pearl2009causality}, but the intuitions that lie behind can be traced back to \cite{lewis1981causal} and \cite{joyce1999foundations}.

\begin{theorem}[\cite{pearl2009causality}, Section 4.1]{\label{causal_ut}}
Let $G$ be a Causal Graphical Model, and its associated distribution $P_G$. Let $C$ be a set of consequences of interest for a decision-maker. If the decision maker faces a causal environment and if the causal graphical model $G$ is known, then the preference relation $\succeq$ satisfies the von Neumann-Morgenstern rationality axioms if and only if:
\begin{eqnarray*}
&a \succeq b&\\
& \textrm{ if and only if }&\\
 &\sum_{c \in C} P_G(c \mid do(a))u(c) \geq \sum_{c \in C} P_G(c \mid do(b))u(c).&
\end{eqnarray*} 
\end{theorem}
Equivalently, the action that must be chosen is 
\begin{equation}
a^\ast = \textrm{argmax}_{a \in \mathcal{A}} \sum_{c \in C} P_G(c \mid do(a))u(c). 
\end{equation}
Further related work \cite{lewis1981causal,heckerman1995decision,joyce1999foundations,dawid2002influence, dawid2012decision,eells2016rational,stern2017interventionist}

\subsection{Main Result}{\label{main_result}}
We now consider \textbf{the case in which a rational decision maker does not know the causal model} which controls her environment. We must add \textbf{a new axiom} to the rationality axioms in order to keep simple our proof: \textbf{Choosing within a Causal Decision Problem corresponds to intervening a variable of the true causal model which controls the environment.}

Using Pearl's result, we saw how a rational decision maker can use a \textit{known} causal model in order to make a choice. Now, since the decision maker does not know the causal model, we argue that she must use \textit{beliefs} about which causal relations hold in her environment, and use them in order to make a choice. In this case, in which a Causal Graphical Model controls the relation between actions and outcomes, any subjective information about the environment must consider causal structures. For this reason, \textbf{we assert that the probability distribution that the decision maker has in mind is in fact a distribution over causal structures}, where the decision maker uses each structure \textit{as if} it were the true one in order to choose the best action within each structure by using Theorem \ref{causal_ut}. We assume a finite set of actions and outcomes.

\begin{theorem}[Main Result]
\label{causal_savage}
In a finite, bounded, Causal Decision Problem \\ $(\mathcal{A}, \mathcal{G},\mathcal{E},\mathcal{C},\succeq)$, where $\mathcal{G}$ is a Causal Graphical Model, we have that the preferences $\succeq$ of a decision maker are rational if and only if there exists a utility function, a probability distribution $P_C$ over a non-empty family $\mathcal{F}$ of causal graphical models such that for each $a,b \in \mathcal{A}$:
\begin{eqnarray*}
&a \succeq b&\\ 
&\textrm{ if and only if }&\\
&\sum_{c \in \mathcal{C}} u(c) \left( \sum_{g \in \mathcal{F}} P_g(c \mid do(a))P_C(g) \right)&\\
&\geq& \\
&\sum_{c \in \mathcal{C}}  u(c) \left( \sum_{g \in \mathcal{F}} P_g(c \mid do(b))P_C(g) \right)&
\end{eqnarray*}
where $P_g$ is the probability distribution associated with the causal structure $g$. 
\end{theorem}

\begin{proof}
Following \citeauthor{bernardo2000bayesian} proof of the Expected Utility Principle (Proposition 2.22, p. 52, \cite{bernardo2000bayesian}), let $a_i = \{ c_{ij} \mid E_{ij}, j = 1 , ... , n \}$ be an option.\\
\\
The existence of a probability $P(\cdot)$, associated with uncertain events $E$, is given by Axiom 5(2). Also, the utility $u(c)$, is the real number $\mu(S)$, associated with any standard event $S$, such that 
\[ c \sim \{ c^\ast \mid S, c_\ast \mid S^c \}. \]
From the coherence Axioms 4(3) and 5(2), and Proposition 2.13 in \cite{bernardo2000bayesian}, there exists events $S_{ij}$ and $S'_{ij}$, such that 
\[ c \sim \{ S_{ij}', c_{\ast} \mid S_{ij}^{'c} \}, S_{ij} \indep E, P(S_{ij}')= P(S_{ij}), \]
Then, by Proposition 2.10 in \cite{bernardo2000bayesian}, $c_{ij} \sim \{ c^\ast \mid S_{ij}, c_{\ast} \mid S^c_{ij} \}$, where $P(S_{ij})=u(c_{ij})$.\\
Now, for any other option $a$ and $i=1,2$,
\[ \{[ c_{ij} \mid E_{ij}  ], j =1,...,n, a \} \sim \{ [(c^\ast \mid S_{ij},c_\ast \mid S^c_{ij}) \mid E_{ij}], j=1,...,n_i, a\}. \]
According to Definitions \ref{causal_environment} and \ref{causal_decision_problem}, and the new axiom added, we can split each event $E_{ij}$ into $\{E'_{ij},g_{ijk},I\}$, where
\[ E'_{ij} = \{ X_1 = x_1 , ... , X_n = x_n \}, \]
where $X_1 , ... , X_n$, are the variables of the CGM that controls the environment $\mathcal{G}$; and $g_{ijk}$, specifies a graphical structure over the set of variables $\{ X_1 , ... , X_n \}$; and $I$ is a pair of indices which indicates which variable has been intervened in order to affect which other variable.\\
\\
We can define $A_i$ as 
\[A_i = \cup_j (E_{ij} \cap S_{ij}),\]
Then, it is shown in \cite{bernardo2000bayesian} that events with this structure hold that:
\[ a_1 \succeq a_2 \Leftrightarrow A_1 \geq A_2 \Leftrightarrow P(A_1) \geq P(A_2).\]
And, we can express $P(E_{ij}\cap S_{ij})$ as:
\[P(E_{ij} \cap S_{ij})=P(E_{ij})P(S_{ij})\]
This means that,
\begin{eqnarray*}
 P(A_i ) &=& P(\cup_j (E_{ij} \cap S_{ij}))\\
         &=&    P(E_{ij})P(S_{ij})\\
         &=&     \sum_j u(c_{ij}) P(E_{ij}).
\end{eqnarray*}
We now use the fact that $E_{ij}=\cup_k \{E',g_{ijk},I\}$, to see that 
\begin{eqnarray*}
P(E_{ij})&=&P(E'_{ij} \cap (g_{ij1} \cup ... \cup g_{ijn_k}) \cap I)\\ 
&=& \sum_k P(E' \cap I \mid g_{ijk})P(g_{ijk}),
\end{eqnarray*}
Where the last term $P(g_{ijk})$, is supported within a family of causal models. Now, this last expression can be further expanded as:
\begin{eqnarray*}
\sum_k P(E'_{ij} \cap I \mid g_{ijk})P(g_{ijk}) &=& \sum_k P(\{X_1 , ... , X_n\} \cap I \mid g_{ijk})P(g_{ijk}) \\
&=& \sum_k P_g(c_{ij} \mid do(x_i))P(g_{ijk}).
\end{eqnarray*}
Therefore,
\begin{eqnarray*} 
P(A_i ) &=& \sum_j u(c_{ij}) P(E_{ij})\\ 
&=& \sum_j u(c_{ij}) \left( \sum_k P_g(c_{ij} \mid do(x_i))P(g_{ijk}) \right). 
\end{eqnarray*}
And the result follows, since $a_1 \succeq a_2 \Leftrightarrow P(A_1) \geq P(A_2).$ 
\end{proof}

\subsection{Interpretation}
Theorem \ref{causal_savage} asserts that a rational decision maker who faces a Causal Decision Problem with unknown causal information must use a probability distribution $P_C$ over a family $\mathcal{F}$ of causal structures, and, within each structure, $g \in \mathcal{F}$, use the term $P_g(c \mid do(a))$ in order to find the probability of obtaining a certain consequence given that the intervention $do(\cdot)$ is performed; in this way, the optimal action $a^\ast$ is given by:
\begin{equation}
a^\ast = \textrm{ argmax }_{a \in \mathcal{A}}  \sum_{c \in \mathcal{C}} u(c) \left( \sum_{g \in \mathcal{F}} P_g(c \mid do(a))P_C(g) \right). 
\end{equation}
We note that $a^\ast$ is obtained by taking into account the utility obtained by every possible consequence weighted using both the probability of causing such action within a specific causal model $g$ and the probability that the decision-maker assign to such $g \in \mathcal{F}$.

We are considering a \textit{normative} interpretation for Theorem \ref{causal_savage} according to which a decision maker must use any causal information in order to obtain the best possible action. Such action must be obtained by considering the \textit{beliefs} of the decision maker about the causal relations that hold in her environment (the distribution $P_C$), how such relations could produce the best action when considered \textit{as if} they were true (distribution $P_g$), and the satisfaction (utility $u$) produced by the consequences of actions.

It is known that humans tend to ignore pure probabilistic information over causal information \citep{tversky1980causal}, and are in fact able to learn, and use, causal models in sequential decision-making processes \citep{lagnado2007beyond,sloman2006causal,nichols2007decision,meder2010observing,hagmayer2013repeated,wellen2012learning}, even though such learning is not perfect \citep{rottman2014reasoning}. Therefore, this theorem provides the basis for a much stronger, and computationally implementable, framework for decision-making in which causal information is used over associative information, even though complete causal information may not be available to the decision-maker. 

\section{Application: Causal Games and Nash Equilibrium}
\label{causal_games}
In this section we study and develop an application of the previous result in the domain of Game Theory: we consider a \textit{strategic game} between $N$ rational players who are situated in a causal environment. A game is a model of a situation in which several players must take an action and afterward they will be affected both by the outcome of their own action as well as the actions of the other players \citep{osborne1994course}.

In a strategic game, it is assumed that no player knows the action taken by any other players; we also assume that the causal mechanism, which is represented by a Causal Graphical Model $\mathcal{G}$, remains fixed and it is unknown for each player. 
In this game, players ignore the actions taken by any other player, and since the causal model which controls the environment is unknown to every player, the players also ignore the information that players will use in order to take their respective actions: strategic games of this type are called \textit{Bayesian Games}, introduced by Harsanyi \citep{harsanyi1967games1,harsanyi1968games2,harsanyi1968games3}.

With this contribution, we expect to show that standard notions of game theory such as Nash Equilibrium can be extended to the case in which causal information is considered over associative information. Therefore, provide motivation to further extend classical results to use causal information as a basis.

In the games we will consider, the uncertainty of every player consists of two levels: on the first level, the true causal model $\mathcal{G}$; on the second level, what an action $do(a)$ causes if a certain Causal Graphical Model $\omega$ is considered to be the causal model. 
\begin{definition}
A Bayesian strategic game \citep{osborne1994course} consists of: A finite set $N$ of players. A finite set $\Omega$ of \textit{states of nature}. For each player, a nonempty set $A_i$ of actions. For each player, a finite set $T_i$ and a function $\tau_i : \Omega \mapsto T_i$ is the signal function of the player. For each player, a probability measure $p_i$ over $\Omega$ such that $p_i (\tau^{-1}_i (t_i))>0$ for all $t_i \in T_i$. A preference relation $\succeq_i$ defined over the set of probability measures over $A \times \Omega$ where $A= A_1 \times \cdots A_n$
\end{definition}

We consider $\Omega$ to be a family of admissible causal models; in this way,  $\omega \in \Omega$ being the true state of Nature fixes a causal model which controls the environment in which the players make their choices. 
In classical Bayesian games, once $\omega \in \Omega$ is realized as the true state, then each player receives a signal $t_i=\tau_i (\omega)$
and the posterior belief $p_i(\omega \mid \tau^{-1}_i (t_i) )$ given by 
$p_i(\omega) / p_i (\tau^{-1}_i (t_i))$
if $\omega \in \tau^{-1}_i (t_i)$. In the case of causal Bayesian games, we must consider both the probability $p_i$ of $\omega$ being the true state as well as the probability $p^\omega_i$ of observing a certain consequence when doing some action $a_i$ if $\omega$ is the true model. Following \cite{osborne1994course}, we define a new game $G^\ast$ in which its players are all of the possible combinations $(i, t_i) \in N \times T_i$, where the possible actions for $(i,T_i)$ is $A_i$. \cite{osborne1994course} show that for a fixed player $i \in N$, the posterior probability 
$p(\omega \mid \tau^{-1}_i (t_i))$
induces a lottery over the pairs 
$(a^\ast(j,\tau_j(\omega)))_j,\omega) \textrm{ for some other $j \in N$.} $
This lottery assigns to $(a^\ast(j,\tau_j(\omega)))_j,\omega)$ the probability
$p_i(\omega) / p_i (\tau^{-1}_i (t_i))$ if $\omega \in \tau^{-1}_i (t_i)$. The classical Bayesian game's Nash Equilibrium is the Nash equilibrium of $G^\ast$ \citep{osborne1994course}. Now, we consider the second level of uncertainty: the consequences caused by some action $a$ through a causal model $\omega \in \Omega$. We notice that the posterior probability itself induces a probability distribution defined over \textit{actions} for each player once a \textit{desired consequence} is fixed, this distribution, according to Theorem \ref{causal_savage} is given by 
$p^\omega_i (c \mid do(a^\ast_i), a^\ast_{-i}) p_i(\omega \mid \tau^{-1}_i (t_i)).$ This motivates the following definition of a \textit{Causal Nash equilibrium}.

\subsection{Causal Nash Equilibrium}
For each player $i \in N$ in the strategic game, we define the following probability distribution over consequences:
\begin{equation}
p^a_i (c) =  p^\omega_i (c \mid do(a_i), a_{-i}) p_i(\omega) \textrm{ for } a \in A=A_1 \times \cdots \times A_N, 
\end{equation}
where $p^\omega_i$ is the probability of causing a certain consequence within a causal structure $\omega$, and $p_i$ is the player's \textit{posterior beliefs} about the causal structure that controls the environment, and $do()$ is the well-known intervention operator from \citep{pearl2009causality}. We now define: 
\begin{equation}
u^C_i (a) = \sum_{c \in C}  u_i(c) p^a_i (c)  \textrm{ for }  a \in A=A_1 \times \cdots \times A_N.
\end{equation}
Notice that $u^C_i$ evaluates an action profile $a \in A$ in terms of: 
The knowledge each player has about the causal structure represented by $p_i$, which allows each player to evaluate the probability of causing outcomes in terms of actions by using the $do$ operator,
as well as the \textit{observed} actions that are taken by the other players, given by $a_{-i}$, and the preferences of each player $u_i$.
Using this new function, we define the equilibrium for a strategic game with causal information and Bayesian players as:
\begin{definition}
A Nash equilibrium for this \textit{causal strategic game} is an action profile $a^\ast \in A$ if and only if
\begin{equation}{\label{causal_nash}}
 u^C_i(a^\ast) \geq u^C_i(a_i, a^\ast_{-i}) \textrm{ for any other } a_i \in A_i. 
 \end{equation}
\end{definition}
This is, an action profile is a Nash equilibrium if and only if each player uses her current knowledge about the causal structure of the environment in order to (causally) produce the best possible outcome given the actions taken by the other players. The existence of the Causal Nash Equilibrium is guaranteed if every $A_i$ is a nonempty compact convex set in some $\mathbb{R}^n$ and if the preference relation induced by $u^C_i$ is continuous and quasi-concave as proved by \cite{osborne1994course}.

\section{Limitations}
\label{limitations}
We are working within the classical rationality assumption. Rationality can be ultimately thought of as a \textit{consistent} or coherent way of making choices, but the precise definition has been a subject of debate. See \cite{ellsberg1961risk,gilboa2009decision} and \cite{machina2014ambiguity} for critiques of the Savage Rationality Axioms. We have favored causal graphical models over other alternatives since it has been argued that several cognitive processes, such as causal reasoning, can be best represented as graphical models \citep{danks2014unifying,sloman2015causality,hagmayer2016causal}.

\section{Summary}
We have defined a Causal Decision Problem in terms of a classical Decision Problem under Uncertainty provided by a causal mechanism that mediates between actions and outcomes; this causal mechanism is assumed to be represented as a causal graphical model. In the case in which a rational decision maker knows such causal relations, we have seen that \citeauthor{pearl2009causality} has provided a decision-making result \cite{pearl2009causality}. 

On the other hand, when a decision maker does not know the causal mechanism, in Theorem \ref{causal_savage} we have provided preference representation result for causal decision making; our result explicitly states how a rational decision maker should use subjective beliefs, encoded as a probability distribution over causal models, as well as the causal inference machinery within the considered causal structures in order to find an optimal action.

With these two cases being covered, we have shown how causal models can provide the action-outcome probabilities for decision-making.

As an application, by using Theorem \ref{causal_savage} and taking as a basis Harsanyi's model of a Bayesian Game in which every player has incomplete information about both the actions taken by other players as well as the information that made each player take his action, we have been able to provide a definition of a Causal Nash Equilibrium (Equation \ref{causal_nash}) in which every player is aware that there exists a causal mechanism that will produce some consequence once he takes an action.

Finally, by extending a known result from Pearl to the case of \textit{unknown} causal information, we have shown how action-outcome probabilities for decision-making can be grounded in causal models. 

\section{Acknowledgements}
MGS thanks Dr. Gerardo Gonzalez Robert for his thoughts and comments.
\newpage
\appendix
\section{Coherence axioms \citep{bernardo2000bayesian}}
\begin{itemize}
\item Axiom 1: \textbf{Comparability of consequences and dichotomized options}: 
    \begin{enumerate}
        \item There exists consequences $c_1,c_2$ such that $c_2 \succ c_1$.
        \item For any consequences $c_1, c_2 \in \mathcal{C}$ and any events $E,F \in \mathcal{E}$, then either $\{c_2 \mid E, c_1 \mid E^c  \} \succeq  \{ c_2 \mid F , c_1 \mid F^c \}$ or $\{ c_2 \mid F , c_1 \mid F^c \}  \succeq \{c_2 \mid E, c_1 \mid E^c  \} $
    \end{enumerate}
\item Axiom 2: \textbf{Transitivity of preferences}:
    \begin{enumerate}
        \item $a \succeq a$.
        \item $a_2 \succeq a_1$, $a_3 \succeq a_2$, $\Rightarrow$ $a_3 \succeq a_1$.
    \end{enumerate}
\item Axiom 3: \textbf{Consistency of preferences}:
    \begin{enumerate}
        \item If $c_2 \succeq c_1$, then for any event $G$ more likely than $\emptyset$, $c_2 \succeq_G c_1$.
        \item If $c_2 \succ c_1$ and $ \{ c_2 \mid F, c_1 \mid F^c \} \succeq \{ c_2 \mid E, c_1 \mid E^c \}$, then $F \succeq E$.
    \end{enumerate}
\item Axiom 4: \textbf{Existence of standard events}: There exists a sub-algebra $\mathcal{S}$ of $\mathcal{E}$ and a function $\mu:\mathcal{S}\to [0,1]$ such that
    \begin{enumerate}
        \item $S_2 \succeq S_1$ if and only if $\mu(S_2) \geq \mu(S_1).$
        \item Disjoint $S_1$ and $S_2$ imply that $\mu(S_1 \cup S_2)= \mu(S_1) + \mu(S_2).$
        \item For any $\alpha \in [0,1]$, and events $E,F \in \mathcal{E}$, there exists a standard event $S \in \mathcal{S}$ such that $\mu(S)=\alpha$ and $E \indep S$ and $F \indep S.$
        \item For $S_1$ and $S_2$, independent standard events, we have that $\mu(S_1 \cap S_2) = \mu(S_1)\mu(S_2).$
        \item If $E \indep S$, $F \indep S$, and $E \indep F$, then $E \sim S \Rightarrow E \sim_F S$. 
    \end{enumerate}
\item Axiom 5: \textbf{Precise measurement}: 
    \begin{enumerate}
        \item If $c_2 \succeq c \succeq c_1$, there exists a standard event $S \in \mathcal{S}$ such that $c \sim \{ c_2 \mid S , c_1 \mid S^c \}$.
        \item For each event $E \in \mathcal{E}$, there exists a standard event $S \in \mathcal{S}$ such that $E \sim S$.
    \end{enumerate}
\end{itemize}

\bibliographystyle{apalike}
\bibliography{Bib.bib}
\end{document}